\pdfoutput=1

\documentclass[11pt]{article}

\usepackage[final]{acl}

\usepackage{times}
\usepackage{latexsym}

\usepackage[utf8]{inputenc} %
\usepackage[T1]{fontenc}    %
\usepackage{hyperref}       %
\usepackage{url}            %
\usepackage{booktabs}       %
\usepackage{amsfonts}       %
\usepackage{nicefrac}       %
\usepackage{microtype}      %
\usepackage{xcolor}         %
\usepackage{listings}
\usepackage{amsmath}

\usepackage{algpseudocode}
\usepackage{algorithm}
\usepackage{float}

\usepackage{hyperref}       %
\usepackage{url}            %
\usepackage{booktabs}       %
\usepackage{amsfonts}       %
\usepackage{nicefrac}       %
\usepackage{xcolor}         %
\usepackage{amsmath}
\usepackage{soul}
\usepackage{caption}
\usepackage{algpseudocode}
\usepackage{algorithm}
\usepackage{tikz}
\usepackage{pgfplots}
\usepackage{graphicx}
\usepackage{subcaption}
\usepackage{cleveref}
\usepackage{multirow}
\usepackage{amssymb}
\usepackage{graphicx}
\usepackage{wrapfig}

\usepackage{times}
\usepackage{latexsym}
\usepackage[T1]{fontenc}
\usepackage[utf8]{inputenc}
\usepackage{microtype}
\usepackage{inconsolata}
\usepackage{amsthm}

\usepackage{mathtools, cuted}
\usepackage{caption}

\usepackage[T1]{fontenc}

\usepackage[utf8]{inputenc}

\usepackage{microtype}

\usepackage{inconsolata}

\usepackage{graphicx}

\newcommand{\bbQ}{\mathbb{Q}}
\newcommand{\expe}{\mathbb{E}}

\newtheorem{theorem}{Theorem}
\newtheorem{assumption}{Assumption}
\newtheorem{lemma}{Lemma}
\newtheorem{definition}{Definition}

\newcommand{\myred}[1]{\textcolor{red!70!black}{#1}}
\newcommand{\myblue}[1]{\textcolor{black!40!blue}{#1}}

\newcommand{\redQ}[2]{\myred{\bbQ_{#1}({#2})}}
\newcommand{\blueQ}[2]{\myblue{\bbQ_{#1}({#2})}}
\newcommand{\what}[1]{\widehat{#1}}

\newcommand\dfa{\myblue{\Delta_A^\mathrm{fwd}}}
\newcommand\dba{\myred{\Delta_A^\mathrm{bwd}}}
\newcommand\dfw{\myblue{\Delta_W^\mathrm{fwd}}}
\newcommand\dbw{\myred{\Delta_W^\mathrm{bwd}}}
\newcommand\dbda{\myred{\Delta_{\nabla A}^\mathrm{bwd}}}

\newcommand\efa{\myblue{\epsilon_A^\mathrm{fwd}}}
\newcommand\eba{\myred{\epsilon_A^\mathrm{bwd}}}
\newcommand\efw{\myblue{\epsilon_W^\mathrm{fwd}}}
\newcommand\ebw{\myred{\epsilon_W^\mathrm{bwd}}}
\newcommand\ebda{\myred{\epsilon_{\nabla A}^\mathrm{bwd}}}

\newcommand{\bw}{\boldsymbol{w}}

\definecolor{codegreen}{rgb}{0,0.6,0}
\definecolor{codegray}{rgb}{0.5,0.5,0.5}
\definecolor{codepurple}{rgb}{0.58,0,0.82}
\definecolor{backcolour}{rgb}{0.95,0.95,0.92}

\newcommand\woqat{\nabla_{\widehat{\bw}} L(\widehat{\bw}, x, y)}

\title{Training with Fewer Bits: Unlocking Edge LLMs Training with Stochastic Rounding}

\author{
  {Taowen Liu},
  {Marta Andronic},
  {Deniz Gündüz},
  {George A. Constantinides}
\\ 
  Department of Electrical and Electronic Engineering\\
  Imperial College London\\
  London, UK
\\ 
\texttt{\{{tony.liu20}, {marta.andronic18}, {d.gunduz}, {g.constantinides}\}@imperial.ac.uk}
}

\begin{document}

\maketitle

\begin{abstract}
LLM training is resource-intensive. Quantized training improves computational and memory efficiency but introduces quantization noise, which can hinder convergence and degrade model accuracy.  Stochastic Rounding (SR) has emerged as a theoretically attractive alternative to deterministic rounding, offering unbiased gradient estimates. However, its interaction with other training factors—especially batch size—remains underexplored. In this paper, we present a theoretical and empirical study of mini-batch stochastic gradient descent (SGD) with SR, showing that \textbf{increased batch sizes can compensate for reduced precision during backpropagation}. Furthermore, we show that quantizing weights and activations impacts gradient variance in distinct ways. Our experiments validate these theoretical insights.
\end{abstract}

\section{Introduction}

Training large language models (LLMs) demands significant computational and memory resources. Mixed-precision training, using lower-precision formats, offers a crucial path to efficiency~\cite{micikevicius2018mixed, das2018mixed, wu2018training, lee2023training}, enabling training on edge devices~\cite{kandala2024yourdata}. However, reduced precision introduces quantization noise, potentially hindering convergence. Stochastic rounding (SR)~\cite{higham2022stochastic} provides an attractive, unbiased alternative to deterministic rounding methods for managing this noise. As shown in Figure~\ref{fig:norm}, SR can maintain training stability and performance where round-to-nearest (RTN) fails, especially under aggressive quantization.

Despite SR's advantages, its practical interaction with mini-batch stochastic gradient descent (SGD), particularly the role of batch size in mitigating SR-induced noise, remains underexplored. Key questions persist: Can larger batch sizes—a common variance reduction tool—effectively counteract SR noise? How does this apply distinctly to quantizing shared model weights versus per-sample activations and gradients? And crucially, how does SR's variance impact SGD convergence guarantees? Addressing these is vital for establishing theoretically grounded guidelines for optimal precision, rounding, and batch size selection, especially to unlock aggressive quantization for training LLMs on edge devices.

Although prior work has explored mixed-precision training empirically~\cite{micikevicius2018mixed, blake2023unit, peng2023fp8lm} and theoretically analyzed aspects like gradient quantization for efficient communication~\citep{xia2021simple, xia2022influence, xia2023convergence, xia2023stochastic} or weight quantization~\cite{li2017training}, a rigorous understanding of low-precision arithmetic within the backpropagation process itself, combined with the specific variance characteristics introduced by SR in a mini-batch context, remains limited. \citet{chen2020statistical} analyzed the impact of quantization error but did not explicitly incorporate batch size as a variable interacting with different sources of SR noise. 

This paper addresses this gap through a theoretical and empirical analysis of SR within mini-batch SGD. We explicitly model and differentiate the statistical properties of noise from quantizing shared weights versus per-sample activations and gradients during backpropagation. Our core theoretical finding, supported by empirical validation, is that \textbf{increased batch sizes can effectively compensate for reduced precision during backpropagation by mitigating SR-induced variance from per-sample operations.} Specifically, we demonstrate that reducing precision in activation and gradient quantization can be offset by a quantifiable increase in batch size—for instance, a 1-bit reduction may be balanced by at most a fourfold batch increase to maintain convergence quality, with practical increases often being milder.

Our main contributions are:

\begin{itemize}
    \item We develop a comprehensive theoretical analysis that explicitly models the impact of mini-batch size on SGD convergence when employing SR for quantizing different components in the training pipeline. This framework distinguishes between noise originating from weight quantization and noise from per-sample activation and gradient quantization.
    \item Leveraging our framework, we prove that the variance introduced by stochastically rounding per-sample activations and their gradients during backpropagation (specifically, for computing weight gradients like $\widehat{A}_{in}^T \widehat{\nabla A}_{out}$ and input activation gradients like $\widehat{\nabla A}_{out} \widehat{W}^T$) decays inversely with the mini-batch size ($1/b$). This provides a theoretical underpinning for using larger batches to counteract reduced precision in these operations.
    \item We conduct experiments on both image classification and LLM fine-tuning to validate our theoretical predictions. These experiments demonstrate SR's superiority over RTN and confirm the predicted batch size scaling effect. From these results, we derive practical guidelines quantifying the trade-off, e.g., showing that a 1-bit precision reduction in activation/gradient quantization can be compensated by approximately doubling the batch size in practice to maintain similar convergence behavior.
\end{itemize}

\section{Related Works}
\label{related work section}

\subsection{Mixed-Precision Training}
Mixed-precision training has emerged as a crucial approach for efficient LLMs training by quantizing weights, activations, and gradients. Typically, general matrix-matrix multiplication (GEMM) operations are performed in mixed-precision during gradient computation. Early research demonstrated the viability of training with reduced precision using \texttt{FP16}~\cite{gupta2015deep, micikevicius2018mixed} and \texttt{INT16}~\cite{das2018mixed}. The field progressed with the development of \texttt{FP8}-based methods~\cite{wang2018training, banner2018scalable, yang2020training} and even 4-bit methods~\cite{sun2020ultra, chmiel2021logarithmic}. Various techniques have been proposed to address the challenges of mixed-precision training, including loss scaling~\cite{micikevicius2018mixed}, precision-aware parameter initialization~\cite{blake2023unit}, and blockwise quantization~\cite{peng2023fp8lm}. 

\subsection{Quantization-Aware Training}
Quantization-aware training (QAT) represents another significant direction, focusing on learning quantized weights during training~\cite{hubara2018quantized}. Researchers have explored diverse approaches, including uniform and non-uniform quantization methods~\cite{zhou2017incremental}, mixed-precision quantization~\cite{dong2019hawq}, and learned quantization strategies~\cite{esser2019learned}.

\begin{figure}[t]
	\centering
        \includegraphics[width=0.9\linewidth]{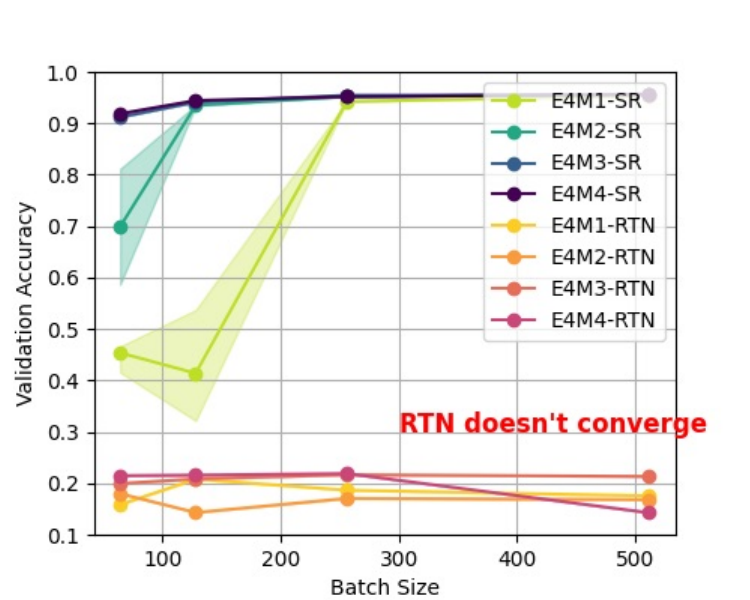}
	\caption{Stochastic rounding (SR) achieves higher accuracy with larger batch sizes, while round-to-nearest (RTN) fails to converge at the same precision. \protect\footnotemark}
        \vspace{-10pt}
	\label{fig:norm}
\end{figure}
\footnotetext{Validation accuracy of WideResNet-16 on CIFAR-10.}

\subsubsection{Stochastic Rounding}

Stochastic rounding (SR)~\cite{higham2022stochastic, xia2021simple, xia2022influence, xia2023convergence, xia2023stochastic} is a probabilistic technique used to mitigate the errors introduced by quantization, especially in low-precision arithmetic by providing unbiased estimates of quantized values. Most theoretical analyses of mixed-precision training have focused on gradient quantization for communication efficiency or storage efficiency~\cite{xia2021simple, xia2022influence, xia2023convergence, xia2023stochastic}. \citet{chen2020statistical} proposed a framework for analyzing the impact of gradient quantization on convergence properties. SR randomly rounds the value to one of the two nearest levels, with probabilities proportional to the distance of the value from each level. SR can be efficiently implemented in hardware using a pseudo-random number generator such as a linear feedback shift register (LFSR). SR has been adopted in modern hardware implementations such as Graphcore processors~\cite{graphcore}. To define SR formally, we first define the threshold quantization function, which can express both round-to-nearest and stochastic rounding.

\begin{definition}[\textbf{Threshold Quantization Function}]
Let $x$ be the number to be quantized, let $\Delta$ be the precision level, and let $\epsilon$ be the threshold. The threshold quantization function is defined as:
    \begin{equation}
        \begin{split}
            \mathbb{Q}_\Delta(x, \epsilon) = \Delta \cdot \begin{cases}
            \lfloor \frac{x}{\Delta} \rfloor,  \text{when } \frac{x}{\Delta} - \lfloor \frac{x}{\Delta} \rfloor < \epsilon \\ 
            \lfloor \frac{x}{\Delta} + 1 \rfloor,  \text{otherwise}
            \end{cases}
        \end{split}
    \end{equation}
    Note: Specifically, when $\Delta = 0$, for any $\epsilon$, define $\mathbb{Q}_0(x, \epsilon) = x$.
\end{definition}

\begin{definition}[\textbf{Round-to-nearest (RTN)}]

Given a real number $x$ and a quantization step size $\Delta$, the RTN quantization of $x$ is defined as: \begin{equation}\hat x = \mathbb{Q}_{\Delta}(x, \frac{1}{2}),\end{equation} where rounding is performed towards positive infinity in the event of a tie.
\end{definition}

\begin{definition}[\textbf{Stochastic Rounding (SR)}]
Given a real number $x$, a quantization step size $\Delta$, and a randomly sampled threshold $\epsilon \sim U[0, 1]$, the SR quantization of $x$ is defined as: \begin{equation}\hat x = \mathbb{Q}_{\Delta}(x, \epsilon).\end{equation} A matrix stochastic rounding is an element-wise stochastic rounding with the thresholds drawn independently for each element.
\end{definition}

SR provides an unbiased estimator of the true arithmetic result, though this accuracy comes at the cost of increased variance compared to deterministic methods. For example, consider summing the value $0.7$ exactly ten times using integer arithmetic. Deterministic RTN consistently rounds $0.7$ up to $1$, yielding a fixed sum of $10$. This deterministic approach introduces a bias of $+3$, as the exact sum should be $7$, but it exhibits zero variance. In contrast, SR probabilistically rounds $0.7$ either up to $1$ with a $70\%$ probability or down to $0$ with a $30\%$ probability. Consequently, repeated summations using SR yield varying outcomes such as $6$, $7$, or $8$. Crucially, the expected sum over many trials converges precisely to the unbiased value of $7$.

\subsubsection{Mixed-Precision Matrix Multiplication}

In LLMs training, the most arithmetic-intensive operation is matrix multiplication in linear or convolution layers.  In accelerators like GPUs, the mixed-precision multiplication performs the scalar multiplication operations in low precision and the accumulation in high precision. Mathematically, this is equivalent to full-precision matrix multiplication of two quantized matrices. $\hat{A} = \mathbb{Q}_{\Delta}(A, \epsilon_A), \hat B = \mathbb{Q}_{\Delta}(B, \epsilon_B)$. Therefore, $AB$ is approximated through $\hat A \hat B$.

\section{A General Framework for Mixed-Precision SGD}

In mixed-precision SGD, matrix multiplications are performed using lower precision. This involves quantizing the layer's input activations, its weight matrix, and the gradients backpropagated from the subsequent layer before they are used in the forward and backward pass matrix multiplications. The choice of quantization parameters (precision levels $\Delta$ and rounding thresholds $\epsilon$) for each of these affects training dynamics. 

\begin{figure*}[t]
\vspace{-1em}
\noindent \begin{minipage}[t]{0.48\textwidth}
\begin{algorithm}[H]
\caption{Mixed-Precision Forward Pass}
\label{algo:forward}
\begin{algorithmic}[1]
\footnotesize
\Function{F}{$A_{in}, W, \Delta_A^{\text{fwd}}, \epsilon_A^{\text{fwd}}, \Delta_W^{\text{fwd}}, \epsilon_W^{\text{fwd}}$}
    \State $\widehat{A}_{in}\gets \blueQ{\Delta_A^{\text{fwd}}}{A_{in}, \epsilon_A^{\text{fwd}}}$
    \State $\widehat{W}\gets  \blueQ{\Delta_W^{\text{fwd}}}{W, \epsilon_W^{\text{fwd}}}$
    \State $A_{out} \gets \widehat{A}_{in} \widehat{W}$ 
    \State \Return $A_{out}$
\EndFunction
\end{algorithmic}
\end{algorithm}
\vspace{-1.2em}
\footnotesize{\textit{Note: Forward and backward passes are deterministic for fixed $\boldsymbol{\epsilon}$. Stochasticity can arise if thresholds $\epsilon$ are sampled, otherwise round-to-nearest is applied.}}
\end{minipage}
\hfill
\begin{minipage}[t]{0.48\textwidth}
\begin{algorithm}[H]
\caption{Mixed-Precision Backward Pass}
\label{algo:back}
\begin{algorithmic}[1]
\footnotesize
\Function{B}{$A_{in}, W, \nabla A_{out}, \Delta_A^{\text{bwd}}, \epsilon_A^{\text{bwd}}, \dots$} %
    \State $\widehat{A}_{in}\gets \redQ{\Delta_A^{\text{bwd}}}{A_{in}, \epsilon_A^{\text{bwd}}}$
    \State $\widehat{W}\gets \redQ{\Delta_W^{\text{bwd}}}{W, \epsilon_W^{\text{bwd}}}$
    \State $\widehat{\nabla}_{A_{out}} \gets \redQ{\Delta_{\nabla A}^{\text{bwd}}}{\nabla A_{out}, \epsilon_{\nabla A}^{\text{bwd}}} $
    \State $\nabla A_{in} \gets \widehat{\nabla}_{A_{out}} \widehat{W}^T$ 
    \State $\nabla W \gets \widehat{A}_{in}^T \widehat{\nabla}_{A_{out}}$ 

    \State \Return $\nabla A_{in}, \nabla W$
\EndFunction
\end{algorithmic}
\end{algorithm}
\end{minipage}
\end{figure*}

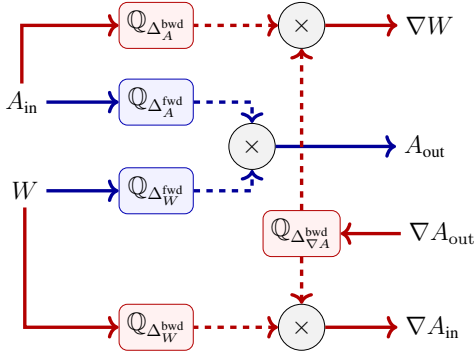
\begin{figure}[h]
	\centering
        \resizebox{0.85\columnwidth}{!}{\usetikzlibrary{positioning, calc, shapes.geometric, arrows.meta}

\begin{tikzpicture}[
  node distance=1.8cm and 1.6cm,
  box/.style={draw, rounded corners=4pt, minimum height=2em, minimum width=3.2em, font=\large, align=center, fill=black!10},
  redbox/.style={box, draw=red!70!black, fill=red!5},
  bluebox/.style={box, draw=blue!70!black, fill=blue!5},
  mult/.style={circle, draw=black, minimum size=2em, inner sep=0pt, font=\large, fill=gray!10},
  labelstyle/.style={font=\large},
]

\node[bluebox] (qin_a) at (0,0) {\(\mathbb{Q}_{\Delta_A^{\text{fwd}}}\)};
\node[bluebox, below=0.7cm of qin_a] (qin_w) {\(\mathbb{Q}_{\Delta_W^{\text{fwd}}}\)};
\node[mult, right=1.2cm of $(qin_a)!0.5!(qin_w)$] (mul) {\(\times\)};
\node[left=1.2cm of qin_a] (ain) {\(\text{\large \(A_{\text{in}}\)}\)};
\node[left=1.2cm of qin_w] (win)  {\(\text{\large \(W\)}\)};
\node[right=2cm of mul] (aout_t) {\(\text{\large \(A_{\text{out}}\)}\)};

\node[redbox, below=1.5cm of qin_w] (qout_w) {\(\mathbb{Q}_{\Delta_W^{\text{bwd}}}\)};
\node[redbox, above=0.5cm of qin_a] (qout_a) {\(\mathbb{Q}_{\Delta_A^{\text{bwd}}}\)};
\node[mult, right=1.4cm of qout_a] (dmul_w) {\(\times\)};
\node[mult, right=1.4cm of qout_w] (dmul_a) {\(\times\)};
\node[redbox, below=2.7cm of dmul_w] (qout) {\(\mathbb{Q}_{\Delta_{\nabla A}^{\text{bwd}}}\)};
\node[right=1cm of qout] (grad_A_out) {\large $\nabla A_{\mathrm{out}}$};

\node[right,  right=1.2cm of dmul_a] (dA_in) {\large \(\nabla A_{\text{in}}\)};
\node[right, right=1.2cm of dmul_w] (dW) {\(\text{\large \(\nabla W\)}\)};

\draw[->, dashed, ultra thick, black!40!blue] (qin_a) -| (mul);
\draw[->, dashed, ultra thick, black!40!blue] (qin_w) -|(mul);
\draw[->, ultra thick, black!40!blue] (mul) -- (aout_t);

\draw[<-, ultra thick, black!40!blue] (qin_a) -- (ain);
\draw[<-, ultra thick, black!40!blue] (qin_w) -- (win);

\draw[->, dashed,  ultra thick, red!70!black] (qout) -- (dmul_a);
\draw[->, ultra thick, red!70!black] (dmul_a) -- (dA_in);
\draw[<-, dashed, ultra thick, red!70!black] (dmul_a) -- (qout_w);
\draw[->, ultra thick, red!70!black] (grad_A_out) -- (qout);

\draw[->, dashed, ultra thick, red!70!black] (qout) -- (dmul_w);
\draw[->, ultra thick, red!70!black] (dmul_w) -- (dW);
\draw[<-, dashed, ultra thick, red!70!black] (dmul_w) -- (qout_a);

\draw[<-, ultra thick, red!70!black] (qout_a) -| (ain);
\draw[<-, ultra thick, red!70!black] (qout_w) -| (win);

\end{tikzpicture}}
	\caption{
             The forward pass is in blue, and the backward pass in red. Solid arrows represent data flow, while dashed arrows indicate the flow of quantized values.
        }
        \vspace{-10pt}
	\label{fig:mix}
\end{figure}

Our analysis focuses on the quantization within linear layers (including convolutional layers viewed as matrix multiplications), as they typically dominate the computational cost \cite{joshi2020essop}. Consider a network with $n$ linear layers, indexed $i=1, \dots, n$. The weights for linear layers are $\bw = {W_1, \dots, W_n}$. The input activation to layer $i$ is $A_{i-1}$, the weight is $W_i$, and the output activation is $A_i$. $x$ and $y$ are the network input and label, respectively. For simplicity, we exclude the bias. We denote $\nabla A_i$ as the gradient of the loss $L$ with respect to $A_i$. We also denote $f(\bw) = \expe_{x,y}[L(\bw, x, y)]$.

For each layer, $5$ quantization operations are involved. The exact quantization operation behavior is exactly defined through quantization thresholds $\epsilon$  and quantization precision $\Delta$. We define $\Delta_i = \{ \dfa, \dfw, \dba, \dbw, \dbda\}$ and $\epsilon_i = \{ \efa, \efw, \eba, \ebw, \ebda \}$. Let $\boldsymbol{\Delta} = \{ \Delta_i \}_{i=1}^n$ and $\boldsymbol{\epsilon} = \{ \epsilon_i \}_{i=1}^n$ represent the collection of all quantization parameters for the entire network. 

\begin{definition}[\textbf{Gradient Approximation}]
The gradient approximation $g(\bw, x, y, \boldsymbol{\Delta}, \boldsymbol{\epsilon})$ for a sample $(x, y)$ is computed via the forward and backward passes described in Algorithms~\ref{algo:forward},\ref{algo:back} and Figure~\ref{fig:mix}.
\end{definition}

\subsection{Mixed-Precision SGD with Deterministic Rounding}

One approach is to use round-to-nearest (RTN) rounding. This corresponds to setting all quantization thresholds $\boldsymbol{\epsilon}$ to 0.5. The approximate gradient is then computed as $g(\bw, x, y,\boldsymbol{\Delta},\boldsymbol{\epsilon}_{0.5})$, where $\boldsymbol{\Delta}$ contains the desired low-precision levels. A drawback of this approach is that deterministic rounding introduces systematic bias. Consequently, the resulting approximate gradient $g(\bw, x, y,\boldsymbol{\Delta},\boldsymbol{\epsilon}_{0.5})$ is usually a \textit{biased} estimator of the true gradient $\nabla_{\bw} L(\bw, x, y)$, unless extreme cases, for example, the intermediate values are exactly lying on the quantization points. This means that even when averaged over the data distribution, the expected approximate gradient may not equal the true expected gradient:
\begin{equation}
\mathbb{E}_{x,y}\left[ g(\bw, x, y,\boldsymbol{\Delta},\boldsymbol{\epsilon}_{0.5})\right]\neq\mathbb{E}_{x,y}\left[\nabla_{\bw} L(\bw, x, y)\right].
\label{eq:deterministic_bias}
\end{equation}

This lack of unbiasedness complicates theoretical convergence analysis, as many standard SGD proofs rely on the assumption of an unbiased gradient estimator.

\subsection{Weight-Only QAT}
\label{sec:standard_qat}

In a weight-only QAT setup~\cite{jacob2018quantization}, the weight is quantized using the same threshold and precision in both forward and backward passes.

In our framework $g(\bw, x, y, \boldsymbol{\Delta}, \boldsymbol{\epsilon})$, this strategy corresponds to the following settings: quantize weight $\widehat{\bw} = \widehat{\bw}^\mathrm{fwd} = \widehat{\bw}^\mathrm{bwd} = \bbQ_{\Delta \bw}(\bw, \epsilon_{\bw})$ and keep the activations/gradients in high precision $\dfa = \dba = \dbda = 0$. While the resulting solution may differ from that obtained with full-precision SGD, it is known that SGD for weight-only QAT can converge near a stationary point \cite{li2017training}. The quantization threshold can either be selected deterministically or stochastically.

Mathematically, the gradient approximation $g$ with this setup is equivalent to evaluating the loss with quantized weights, and then computing gradients with respect to the quantized weights. 
\begin{equation}
    g(\bw, x, y, \boldsymbol{\Delta}, \boldsymbol{\epsilon}) = \nabla_{\what \bw_t}L(\what \bw_t, x, y)\label{eq:qat_update}
\end{equation}

It is also worth noting that only weights are quantized; the potential for hardware acceleration from mixed-precision units is limited. This motivates our exploration of stochastic rounding for other components in the following section.

\subsection{Stochastic Rounding Mixed-Precision SGD for QAT Objectives}
\label{sec:sr_mp_sgd_qat}

Our goal is to compute an estimate $g(\bw, x, y, \boldsymbol{\Delta}, \boldsymbol{\epsilon})$ of the QAT gradient $\woqat$ such that the estimate is unbiased and performing computation in low precision. SR can be employed in the backward pass~\cite{chen2020statistical}. 

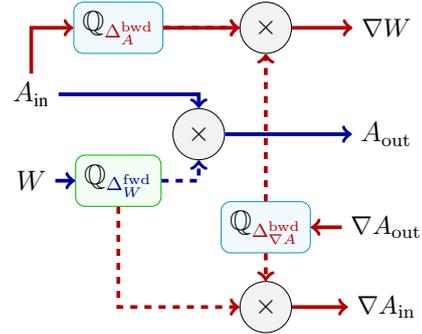
\begin{figure}[h]
	\centering
        \resizebox{0.75\columnwidth}{!}{\usetikzlibrary{positioning, calc, shapes.geometric, arrows.meta}

\begin{tikzpicture}[
  node distance=1.8cm and 1.6cm,
  box/.style={draw, rounded corners=4pt, minimum height=2em, minimum width=3.2em, font=\large, align=center, fill=black!10},
  redbox/.style={box, draw=cyan!70!black, fill=cyan!5},
  bluebox/.style={box, draw=green!70!black, fill=green!5},
  mult/.style={circle, draw=black, minimum size=2em, inner sep=0pt, font=\large, fill=gray!10},
  labelstyle/.style={font=\large},
]

\node (ain) at (0,0) {\(\text{\large \(A_{\text{in}}\)}\)};
\node (w) at (0, -1.3) {\large $W$};
\node[bluebox] (qw) at (1.3, -1.3) {$\bbQ _{\dfw}$};
\node[mult] (mulaout) at (2.5, -0.6) {$\times$};
\node (aout) at (5.3, -0.6) {\large $A_{\mathrm{out}}$};
\node[mult] (muldw) at (3.5, 1) {$\times$};
\node[mult] (muldain) at (3.5, -3.2) {$\times$};
\node (daout) at (5.3, -2) {$\nabla A_{\mathrm{out}}$};
\node (dain) at (5.3, -3.2) {$\nabla A_{\mathrm{in}}$};
\node (dw) at (5.3, 1) {$\nabla W$};
\node[redbox] (qdaout) at (3.5, -2) {$\bbQ_{\dbda}$};
\node[redbox] (qainb) at (1.3,1) {$\bbQ_{\dba}$};

\draw[->, ultra thick, black!40!blue] (w) -- (qw);
\draw[->, dashed, ultra thick, black!40!blue] (qw) -| (mulaout);
\draw[->, ultra thick, black!40!blue] (ain) -| (mulaout);
\draw[->, ultra thick, black!40!blue] (mulaout) -- (aout);

\draw[->, ultra thick, red!70!black] (ain) |- (qainb);
\draw[->, dashed, ultra thick, red!70!black] (qainb) -- (muldw);
\draw[->, ultra thick, red!70!black] (qainb) -- (muldw);
\draw[->, ultra thick, red!70!black] (muldw) -- (dw);

\draw[->, dashed, ultra thick, red!70!black] (qw) |- (muldain);
\draw[->, ultra thick, red!70!black] (muldain) -- (dain);

\draw[->, ultra thick, red!70!black] (daout) -- (qdaout);
\draw[->, dashed, ultra thick, red!70!black] (qdaout) -- (muldain);
\draw[->, dashed, ultra thick, red!70!black] (qdaout) -- (muldw);

\end{tikzpicture}}
	\caption{
             Stochastic Rounding Mixed-Precision SGD for QAT Objectives: Weight quantization is shared across the forward and backward passes. Activation is kept with high precision in the forward pass.
        }       
        \label{img:srmpsgd}
        \vspace{-10pt}
\end{figure}

To achieve this, weights in both forward and backward pass, as well as activations/gradients in backward pass, are quantized in low precision $\dfw = \dbw = \dba = \dbda = \Delta$. Forward pass activation remains high precision to ensure convergence $\dfa = 0$. Activations/gradients are stochastically quantized; $\eba$ and $\ebda$ are sampled from $U[0,1]$. Weight quantization threshold $\epsilon_{\bw}$ can be either stochastic or deterministic, which is discussed later in section \ref{section:bias-from-weight}. Denote $(\boldsymbol{\Delta}^\ast, \boldsymbol{\epsilon}^\ast)$ to be the quantization parameters. This approach is captured by figure \ref{img:srmpsgd}. Under these conditions, the resulting approximate gradient $g(\bw, x, y, \boldsymbol{\Delta}^\ast, \boldsymbol{\epsilon}^\ast)$ is an unbiased estimator of $\woqat$:

\vspace{-15pt}

\begin{equation}
\mathbb{E}_{\eba, \ebda}\left[ g(\bw, x, y, \boldsymbol{\Delta}^\ast, \boldsymbol{\epsilon}^\ast)\right] = \woqat.
\label{eq:sr_unbiased_qat_grad_single_sample}
\end{equation}
Consequently, when taking expectation over the data distribution, the expected approximate gradient equals the expected QAT gradient:
\begin{equation}
\begin{split}
& \mathbb{E}_{x,y,\eba,\ebda}\left[ g(\bw, x, y, \boldsymbol{\Delta}^\ast, \boldsymbol{\epsilon}^\ast)\right] \\ 
= & \mathbb{E}_{x,y}\left[\woqat\right].
\end{split}
\label{eq:sr_unbiased_qat_grad_full_expectation}
\end{equation}

For mini-batch mixed-precision SGD, it corresponds to randomly sampling $\eba$ and $\ebda$ for each sample $j$. However, the same copy of weight is used across a mini-batch, whether it is stochastically or deterministically quantized. We introduce the following simplified notation:

\begin{definition}[\textbf{Stochastic Rounding Mini-batch Mixed-precision SGD}]
With a small abuse of notation, define a simplified notation for Stochastic Rounding Mini-batch Mixed-precision SGD
\begin{equation}
\begin{split}
    \quad \tilde G(\bw) = \frac{1}{b} \sum_{j=1}^b g(& \bw, x^j, y^j, \boldsymbol{\Delta}^\ast, \\ 
    & \{ \efw, \ebw, \eba^{\boldsymbol{j}}, \ebda^{\boldsymbol{j}} \}).
\end{split}
\end{equation}
\end{definition}
To emphasize, $\efw$ and $\ebw$ do not have a superscript $j$, whereas $\eba^{\boldsymbol{j}}$ and $ \ebda^{\boldsymbol{j}}$ have superscript $j$. For each mini-batch, the estimator $\tilde{G}(\bw)$ remains an unbiased estimator of the mini-batch weight only QAT gradient:
\begin{equation}
\label{eq:expected_grad_estimator}
\begin{split}
       & \expe_{(x,y) \in \mathrm{batch},\eba,\ebda}[\tilde {G} (\bw_t)] \\ 
    =  & \mathbb{E}_{(x,y)}[ \nabla_{\what{\bw}_t} L(\what{\bw}_t, x, y)] = \nabla_{\what \bw}f(\what \bw).
\end{split}
\end{equation}

\section{Convergence Analysis}
\label{sec:convergence_analysis}

This section presents a theoretical analysis of the convergence properties of mini-batch mixed-precision SGD in a non-convex setting. We aim to understand how weight quantization and activation/gradient quantization interact with the mini-batch size and hence affect convergence.

The SGD update rule is given by:
\begin{equation}
    \bw_{t+1} = \bw_t - \eta_t \tilde{G} (\bw_t),
\end{equation}
where $\eta_t$ is the learning rate at iteration $t$, and $\tilde{G}(\bw_t)$ is the mini-batch gradient estimator computed using low-precision arithmetic.

One aspect of our analysis is the nature of this estimator. We denote $\expe_t[\cdot] = \expe[\cdot | \bw_t]$ as the expectation conditional on $\bw_t$. As stated in Equation~\ref{eq:expected_grad_estimator}, $\expe_t[\tilde{G}(\bw_t)] = \nabla_{\what{\bw}_t} f(\what{\bw}_t)$, meaning $\tilde{G}(\bw_t)$ is an unbiased gradient estimator of the QAT objective $f(\what{\bw})$. However, since $\nabla_{\what \bw_t}f(\what \bw_t)$ may differ from the true gradient $\nabla f(\bw_t)$, our SGD algorithm operates with a potentially biased estimate of the gradient of our ultimate objective.

Another crucial aspect of our analysis involves quantifying the total variance of this estimator $\tilde{G}(\bw_t)$ and understanding how its components behave, particularly in relation to the mini-batch size $b$. This variance arises from two primary sources: sampling variance as well as variance introduced by stochastic rounding. Our analysis assumes that quantization errors are i.i.d. per-sample. This is a reasonable assumption in our framework, as stochastic rounding is performed independently on each number before multiplication, and accumulation is performed in high precision. A key insight is that SR variance, similar to sampling variance, diminishes as the mini-batch size $b$ increases. Intuitively, because SR is applied independently to each of the $b$ samples' activations/gradients, the errors introduced by these quantization steps tend to average out across the mini-batch. Lemma~\ref{lem:quant_error_scaling} provides a concrete derivation of this $1/b$ scaling.

The interplay between the bias and the two sources of variance (sampling and per-sample SR, both influenced by batch size) is the core of our convergence analysis. We operate on assumptions:

\begin{assumption}[\textbf{Smoothness and Boundedness}]
\label{assum:smoothness}
The true loss function $L(\bw, x, y)$ is $\mathcal{L}$-smooth, meaning its gradient is $\mathcal{L}$-Lipschitz continuous:
$\|\nabla L(\bw, x, y) - \nabla L(\mathbf{v}, x, y)\| \le \mathcal{L} \|\bw - \mathbf{v}\|$ for all $\bw, \mathbf{v}$.
Furthermore, $L(\bw, x, y)$ is bounded below by $L_{\min}$, i.e., $L(\bw, x, y) \ge L_{\min}$.
\end{assumption}

\subsection{Bias from weight quantization}
\label{section:bias-from-weight}

The use of quantized weights $\what{\bw}$ instead of full-precision weights $\bw$ when defining the target gradient introduces a systematic bias. The following lemma bounds this bias.

\begin{lemma}[\textbf{Bounded Gradient Bias from Weight Quantization}]
\label{lemma:bias_BW}
Let $\what \bw=\bbQ_{\Delta W}(\bw,\epsilon_W)$ with quantization step $\Delta W$. The difference between the QAT gradient $\nabla_{\what{\bw}} L(\what{\bw}, x, y)$ and the true gradient $\nabla L(\bw, x, y)$ is uniformly bounded:
\begin{equation}
    \|\nabla_{\what{\bw}} L(\what{\bw}, x, y) - \nabla L(\bw, x, y)\| \le B_W,
\end{equation}
where $B_W =\frac{1}{2} \mathcal{L} \sqrt{d} \Delta_W$ with RTN and $B_W = \mathcal{L} \sqrt{d} \Delta_W$ with SR.
\end{lemma}
\begin{proof}
The proof relies on the $L$-smoothness of $f(\bw)$ and the bound on element-wise quantization error, $\|\what{\bw} - \bw\| \propto \sqrt{d} \Delta_W$. A detailed derivation is provided in Appendix~\ref{proof:lemma:bias_BW}.
\end{proof}

Lemma~\ref{lemma:bias_BW} establishes that the gradient bias $B_W$ is proportional to the weight quantization precision $\Delta_W$. This bias term $B_W$ will contribute to an error floor in our final convergence bound that is \textbf{not reducible by increasing the mini-batch size $b$}.

\subsection{Variance Reduction via Mini-Batching for Per-Sample Quantization}
\label{subsec:variance_reduction_activation_grad}

In addition to the bias, the gradient estimator $\tilde{G}(\mathbf{w}_t)$ is subject to variance. This variance stems from both the stochastic sampling of data and the stochastic rounding. In this section, we focus on the computation $\nabla W = A^T A^{\mathrm{out}}$ since this would provide insight about why the variance of gradient estimator would decay by $\frac{1}{b}$, whereas the variance of $\nabla A = A^{\mathrm{out}} W^T$ is independent from batch size.

Consider the computation of a gradient component $\nabla W_{ij}$ for a weight matrix. Let $A \in \mathbb{R}^{D \times h_1}$ and $A^{\mathrm{out}} \in \mathbb{R}^{D \times h_2}$ be the full-dataset activations and upstream gradients, respectively. The $(i,j)$-th component of the true full-batch gradient is:
\begin{equation}
\nabla W_{ij}^{\text{(full)}} \;=\; \frac{1}{D}\sum_{k=1}^{D} A_{ki} A^{\mathrm{out}}_{kj}.
\end{equation}

For a mini-batch of $b$ samples, the corresponding component of the quantized gradient estimate is:
\begin{equation}
\begin{split}
    & \widehat{\nabla} W_{ij}^{\text{(quant-mini)}} \\ 
   =\; & \frac{1}{b}\sum_{k\in batch} \mathbb{Q}_{\Delta_A}(A_{ki}, \epsilon_{A_{ki}}) \, \mathbb{Q}_{\Delta_{A^{\mathrm{out}}}}(A^{\mathrm{out}}_{kj}, \epsilon_{A^{\mathrm{out}}_{kj}}),
\end{split}
\end{equation}
where $\epsilon_{A_{ki}}$ and $\epsilon_{A^{\mathrm{out}}_{kj}}$ are per-sample random thresholds for stochastic rounding. 

The following lemma characterizes the Mean Squared Error (MSE) of this quantized mini-batch estimate.

\begin{lemma}[\textbf{Error Decomposition for Fully Quantized Gradient Component}]
\label{lem:q_mult_error_fully_quantized}
Under assumption of stochastic rounding, the MSE of $\widehat{\nabla} W_{ij}^{\text{(quant-mini)}}$ with respect to $\nabla W_{ij}^{\text{(full)}}$ is:
\begin{equation}
\begin{split}
       & \mathbb{E}\!\left[\left(\nabla W_{ij}^{\text{(full)}} - \widehat{\nabla} W_{ij}^{\text{(quant-mini)}}\right)^2\right] =  T^{\text{s}}_{ij} + T^{\bbQ}_{ij},
\end{split}
\end{equation}
where 

\begin{equation}
\resizebox{0.98\linewidth}{!}{$
\begin{aligned}
T^{\mathrm{s}}_{ij} &= \mathbb{E}\Big[(\nabla W_{ij}^{\text{(full)}} - \tfrac{1}{b}\!\sum_{k \in \text{batch}} A_{ki}A^{\mathrm{out}}_{kj})^2\Big] \\
T^{\mathbb{Q}}_{ij} &= \mathbb{E}\Big[\tfrac{1}{b}\!\sum_{k\in \text{batch}} \big( A_{ki}\nabla A^{\mathrm{out}}_{kj} - \mathbb{Q}(A_{ki})\, \mathbb{Q}(\nabla A^{\mathrm{out}}_{kj}) \big)^2 \Big].
\end{aligned}
$}
\end{equation}

\end{lemma}
The decomposition holds because the cross-term is zero due to the unbiased nature of stochastic rounding when conditioned on the mini-batch data. (See Appendix~\ref{app:proof_lemma_error_decomp} for details).
Standard SGD tells us that the sampling error $T^{\text{s}}_{ij}$ can be reduced by increasing the batch size $b$. We now show that the quantization error $T^{\bbQ}_{ij}$ also benefits from larger batch sizes.

\begin{lemma}[\textbf{Quantization Error under Per-Sample SR Scaling and Batch Size}]
\label{lem:quant_error_scaling}
Let $T^{\bbQ}_{ij}$ be the quantization error term defined in Lemma~\ref{lem:q_mult_error_fully_quantized}. 
Assuming i.i.d.\ samples within each mini-batch and independent stochastic rounding for every element, 
\begin{equation} \label{eq:T_quant_A_Aout_result}
\resizebox{0.98\linewidth}{!}{$
\begin{split}
 T^{\bbQ}_{ij}   & \leq \; \frac{1}{b} \left( \mathbb{E}[A_{ki}^2]\sigma_{A^{\mathrm{out}}}^2 + \mathbb{E}[(A^{\mathrm{out}}_{kj})^2]\sigma_{A}^2 + \sigma_{A}^2\sigma_{A^{\mathrm{out}}}^2 \right) \\ 
 & \propto \frac{C}{b} \cdot 2^{-2 B}.
\end{split}
$}
\end{equation}

\end{lemma}

\noindent
Here $\sigma_{A}^2$ and $\sigma_{A^{\mathrm{out}}}^2$ denote the quantization error variances of $A$ and $A^{\mathrm{out}}$. 
For step size $\Delta_X$, we have $\sigma_X^2 \leq \Delta_X^2$. 
With $B$ mantissa bits, $\Delta \propto 2^{-B}$, leading to $\sigma_X^2 \propto 2^{-2B}$. 
The constant $C$ depends on the second moments of $A$ and $A^{\mathrm{out}}$. 
A detailed proof is provided in Appendix~\ref{app:proof_lemma_quant_error_scaling}.

Lemma \ref{lem:quant_error_scaling} demonstrates that the SR variance component $T^{\bbQ}_{ij}$ \textbf{decays inversely with the mini-batch size $b$}. This insight is crucial for understanding the trade-off depicted in Figure~\ref{fig:bounds}. Specifically, if reducing precision by 1 bit causes the $2^{-2B}$ factor to increase by $4$ times, this increase can be counteracted by increasing the batch size $b$ by a factor of $4$ to maintain the same level of $T^{\bbQ}_{ij}$. This provides a theoretical basis for the empirical observation that larger batch sizes can compensate for reduced precision when using SR for activations and gradients.

\begin{figure}[t]
    \vspace{-\intextsep} %
    \centering
    \includegraphics[width=0.95\linewidth]{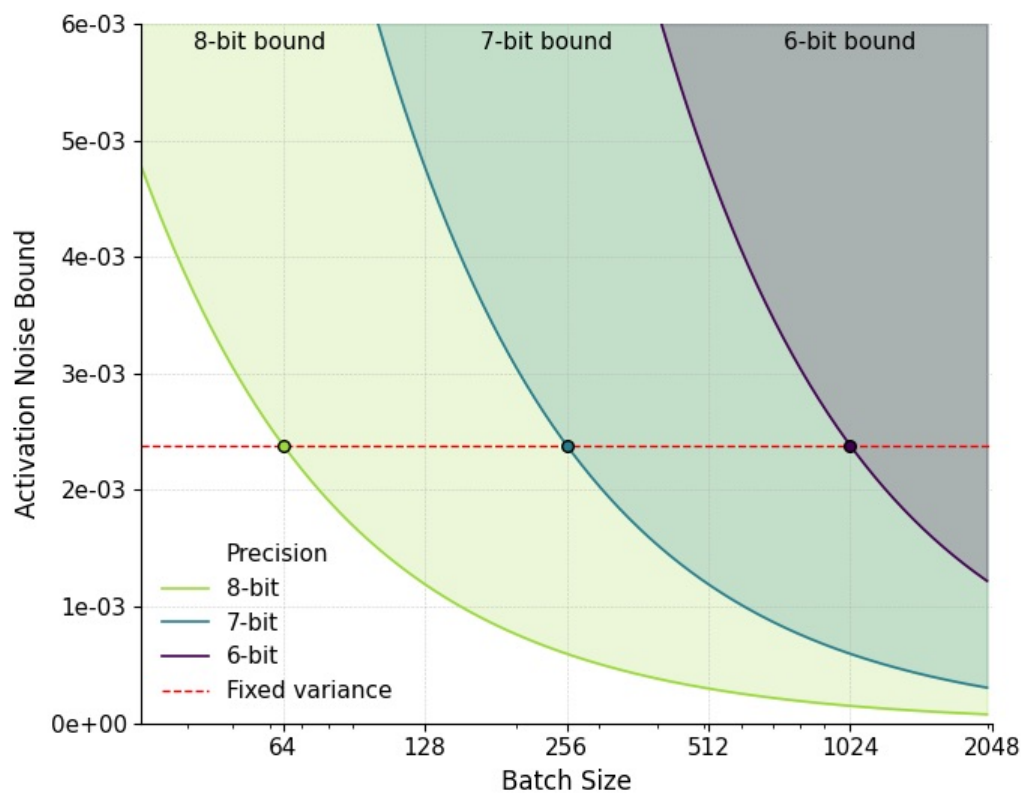} %
    \caption{To guarantee the same variance, increase batch size by at most $4\times$ when reducing $1$ bit of precision.}
    \label{fig:bounds}
    \vspace{-\intextsep} %
\end{figure}

\subsection{Convergence Theorem}
\label{subsec:conv_theorem}

\begin{assumption}[\textbf{Bounded Gradient Variance Components}]
\label{assum:variance_components}
\mbox{}\\

Let $\widehat{G}(\bw_t) = \expe_{(x,y)}[\nabla_{\what{\bw}_t} L(\what{\bw}_t, x, y)]$ be the expected QAT gradient. We assume the following for the stochastic gradient estimator $\tilde{G}(\bw_t)$:

\textbf{Sampling Variance:} The variance of the true QAT gradients is bounded by $\sigma_S^2$:
    
\begin{align}
\expe_{(x,y)} \| \nabla_{\what{\bw}_t} L(\what{\bw}_t, x, y) - \widehat{G}(\bw_t) \|^2 \le \sigma_S^2.
\end{align}
    
\textbf{Quantization Noise Variance:} The expected variance from stochastic rounding for a single sample, when estimating $\nabla_{\what{\bw}_t} L(\what{\bw}_t, x, y)$ with $g(\bw_t, x, y, \boldsymbol{\Delta}^\ast, \boldsymbol{\epsilon}^\ast)$, is bounded by $\sigma_Q^2$:

\begin{align}
\mathbb{E}_{x,y,\boldsymbol{\epsilon}^\ast} & \Bigl[
   \| g(\bw_t, x, y, \boldsymbol{\Delta}^\ast, \boldsymbol{\epsilon}^\ast)
     - \nabla_{\what{\bw}_t} L(\what{\bw}_t, x, y) \|^2 \nonumber
\Bigr] \\
  & \le \sigma_Q^2.
\end{align}

The variance of the mini-batch gradient estimator $\tilde{G}(\bw_t)$ around $\widehat{G}(\bw_t)$ is bounded by:
$$ \expe_t[\|\tilde{G}(\bw_t) - \widehat{G}(\bw_t)\|^2] \le \frac{\sigma_S^2 + \sigma_Q^2}{b},$$
where $b$ is the mini-batch size.
\end{assumption}

With the above setup and assumptions, we can state the convergence properties of the mixed-precision SGD algorithm.

\begin{theorem}[\textbf{Convergence of SGD with Low-Precision Gradients}]
\label{thm:main_convergence}
Under Assumptions \ref{assum:smoothness} and \ref{assum:variance_components}, running low-precision SGD with a constant learning rate $\eta \le 1/4\mathcal{L}$ for $T$ iterations, the average squared norm of the true gradient is bounded by:
\begin{equation}\label{eq:convergence_bound_final}
\begin{split}
 \frac{1}{T} \sum_{t=0}^{T-1} \expe[\|\nabla L(\bw_t)\|^2] 
 \le\;  \frac{4(L(\bw_0) - L_{\min})}{\eta T} \\
 + C_B B_W^2 + C_V \eta \mathcal{L} \left(\frac{\sigma_S^2 + \sigma_Q^2}{b}\right),
\end{split}
\end{equation}
where $C_B = (2+4\eta \mathcal{L})$ and $C_V = 2$ are constants.
\end{theorem}

\begin{figure*}[t] %
    \centering %
    \begin{subfigure}[b]{0.47\linewidth} %
        \centering %
        \includegraphics[width=\linewidth]{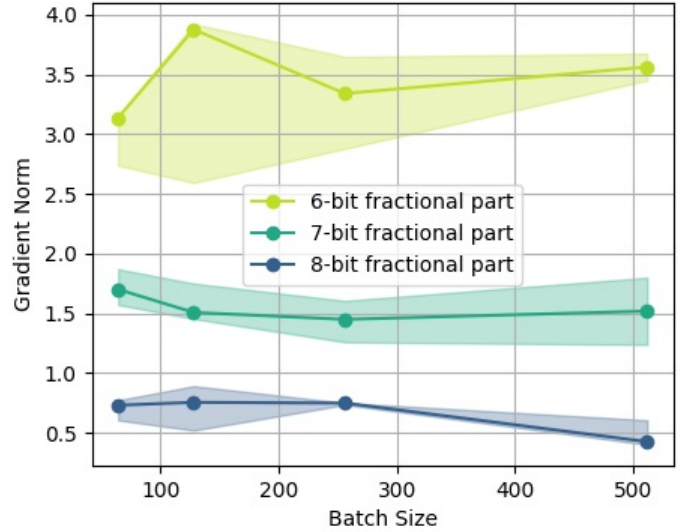} %
        \caption{Gradient norm when trained with different batch sizes. The backward activations/gradients are RTN quantized. CIFAR-10, WideResNet-16\cite{zagoruyko2016wide}.}
        \label{fig:side_by_side_sub1}
    \end{subfigure}
    \hfill
    \begin{subfigure}[b]{0.47\linewidth} 
        \centering
        \includegraphics[width=\linewidth]{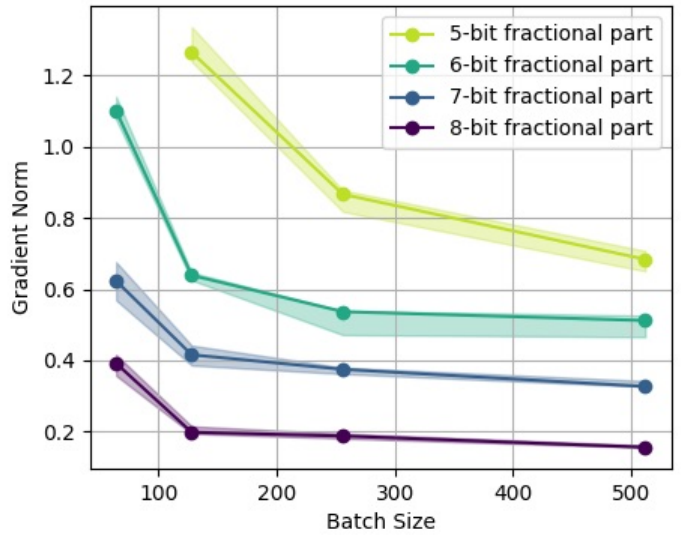} %
        \caption{Gradient norm when trained with different batch sizes. The backward activations/gradients are SR quantized. CIFAR-10, WideResNet-16.}
        \label{fig:side_by_side_sub2}
    \end{subfigure}
    \caption{Practical experiments with image models. The error bars in (a) and (b) represent the 25th-75th percentiles across independent runs.}
    \label{fig:three_side_by_side_main}
\end{figure*}

The term $\frac{4(L(\bw_0) - L_{\min})}{\eta T}$ diminishes as $T \to \infty$, indicating that the algorithm converges to a region where the expected squared gradient norm $\|\nabla L(\bw_t)\|^2$ is bounded by an error floor. The error from the difference between the QAT gradient $\nabla_{\what \bw} L(\what \bw)$ and the true gradient $\nabla L(\bw)$ is not reduced by increasing the mini-batch size $b$. The magnitude of $B_W$ is primarily determined by the precision of weight quantization $\Delta_W$. The variance term $C_V \eta L (\sigma_S^2 + \sigma_Q^2)/b$ captures the noise from two sources: data sampling variance ($\sigma_S^2/b$) and stochastic rounding variance. This entire variance contribution is inversely proportional to the mini-batch size $b$. If activation/gradient precision is reduced, $\sigma_Q^2$ increases (e.g., quadrupling if precision is reduced by 1 bit, as $\sigma_Q^2 \propto 2^{-2B}$). \Cref{thm:main_convergence} shows that increasing the batch size $b$ can directly compensate for this increase in $\sigma_Q^2$, keeping the term $(\sigma_S^2 + \sigma_Q^2)/b$ constant or even reducing it. This is a key mechanism for enabling aggressive quantization of activations/gradients.

\section{Experiment}
\label{sec:experiments}

\begin{figure}[t!]
\centering
        \includegraphics[width=\linewidth]{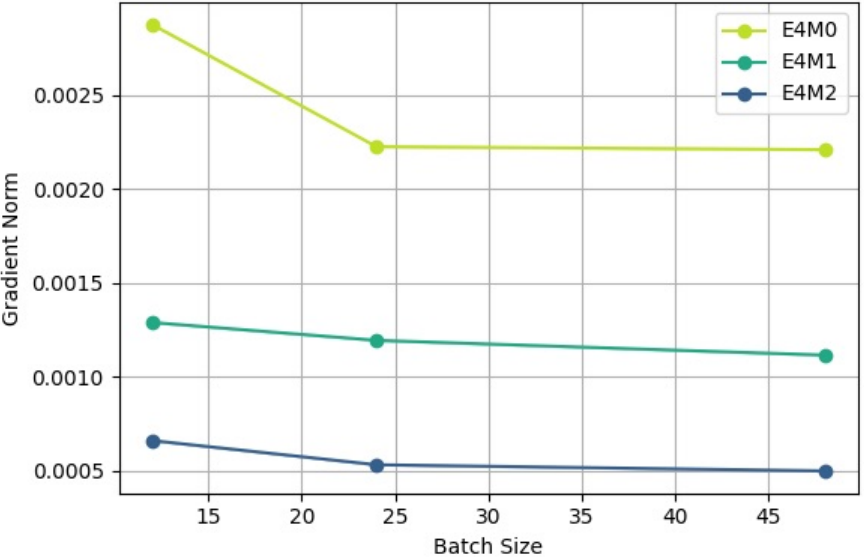}
        \caption{Gradient norm when trained with different batch sizes. The backward activations/gradients are SR quantized. LMSYS-chat, Llama-3.2-3B.}
        \label{fig:side_by_side_sub3}
\end{figure}

Our experiments evaluate both gradient norms and downstream task performance, allowing us to verify the theoretical findings and examine their practical implications. Figure~\ref{fig:three_side_by_side_main} reports results on CIFAR-10 and figure~\ref{fig:side_by_side_sub3} reports results on LMSYS-Chat~\cite{zheng2023lmsys}, showing that the degradation caused by SR in activation and gradient quantization can be substantially mitigated by increasing the batch size, whereas RTN does not benefit from such scaling.

For downstream evaluation, we fine-tuned Llama-3.2-3B~\cite{dubey2024llama} on GSM8K~\cite{cobbe2021training} using both Adam and SGD, and BERT-base-uncased~\cite{devlin2019bert} on the GLUE benchmark. The GSM8K results (Tables~\ref{tab:adam_results} and \ref{tab:sgd_results}) are consistent with our theoretical predictions, showing clear accuracy improvements as the batch size increases across quantization formats. The GLUE~\cite{wang2018glue} benchmark results (Table~\ref{tab:glue_results}) further confirm this trend, with consistent gains across tasks. Taken together, our theoretical findings—validated by empirical evidence—provide actionable guidelines for training under tight resource constraints.

\begin{table}[h]
\centering
\caption{Llama-3.2B fine-tuning on GSM8K with Adam. Reported numbers are end-of-training accuracy.}
\label{tab:adam_results}
\footnotesize
\renewcommand{\arraystretch}{0.9}
\setlength{\tabcolsep}{6pt}
\begin{tabular}{lccc}
\toprule
\textbf{Quant / Batch Size} & \textbf{8} & \textbf{16} & \textbf{32} \\
\midrule
\texttt{E4M0} & 0.210 & 0.265 & 0.290 \\
\texttt{E4M1} & 0.245 & 0.275 & 0.345 \\
\texttt{E4M2} & 0.260 & 0.305 & 0.340 \\
\bottomrule
\end{tabular}
\end{table}

\begin{table}[h]
\centering
\caption{Llama-3.2B fine-tuning on GSM8K with SGD. Reported numbers are end-of-training accuracy.}
\label{tab:sgd_results}
\footnotesize
\renewcommand{\arraystretch}{0.9}
\setlength{\tabcolsep}{6pt}
\begin{tabular}{lcc}
\toprule
\textbf{Quant / Batch Size} & \textbf{8} & \textbf{32} \\
\midrule
\texttt{E4M0} & 0.180 & 0.310 \\
\texttt{E4M1} & 0.290 & 0.390 \\
\texttt{E4M2} & 0.315 & 0.365 \\
\bottomrule
\end{tabular}
\end{table}

\begin{table}[h!]
\centering
\caption{Fine-tuning BERT-base-uncased on GLUE benchmark tasks. We observe a general trend of improved scores as the batch size increases. \\ \dag Matthews correlation coefficient. \ddag Pearson/Spearman correlation coefficients.}
\label{tab:glue_results}
\footnotesize
\resizebox{0.49\textwidth}{!}{
\begin{tabular}{lcccccc}
\toprule
\textbf{Task} & \textbf{Quant} & \textbf{16} & \textbf{32} & \textbf{64} & \textbf{128} & \textbf{256} \\
\midrule
\textbf{CoLA\dag} & \texttt{E4M0} & 0.332 & 0.442 & 0.447 & 0.485 & 0.528 \\
& \texttt{E4M1} & 0.341 & 0.501 & 0.404 & 0.560 & 0.564 \\
& \texttt{E4M2} & 0.408 & 0.443 & 0.539 & 0.580 & 0.618 \\
\midrule
\textbf{MRPC} & \texttt{E4M0} & 0.815 & 0.820 & 0.840 & 0.853 & 0.862 \\
& \texttt{E4M1} & 0.825 & 0.826 & 0.847 & 0.853 & 0.861 \\
& \texttt{E4M2} & 0.846 & 0.856 & 0.879 & 0.883 & 0.908 \\
\midrule
\textbf{MNLI} & \texttt{E4M0} & 0.619 & 0.647 & 0.683 & 0.717 & 0.745 \\
& \texttt{E4M1} & 0.652 & 0.689 & 0.712 & 0.755 & 0.782 \\
& \texttt{E4M2} & 0.674 & 0.731 & 0.755 & 0.777 & 0.793 \\
\midrule
\textbf{QNLI} & \texttt{E4M0} & 0.766 & 0.780 & 0.820 & 0.827 & 0.826 \\
& \texttt{E4M1} & 0.797 & 0.837 & 0.833 & 0.846 & 0.869 \\
& \texttt{E4M2} & 0.818 & 0.840 & 0.862 & 0.868 & 0.898 \\
\midrule
\textbf{QQP} & \texttt{E4M0} & 0.732 & 0.743 & 0.768 & 0.748 & 0.806 \\
& \texttt{E4M1} & 0.755 & 0.768 & 0.793 & 0.808 & 0.820 \\
& \texttt{E4M2} & 0.773 & 0.856 & 0.799 & 0.822 & 0.835 \\
\midrule
\textbf{RTE} & \texttt{E4M0} & 0.553 & 0.446 & 0.511 & 0.525 & 0.554 \\
& \texttt{E4M1} & 0.504 & 0.489 & 0.504 & 0.482 & 0.468 \\
& \texttt{E4M2} & 0.540 & 0.504 & 0.518 & 0.489 & 0.554 \\
\midrule
\textbf{STS-B\ddag} & \texttt{E4M0} & 0.757 & 0.816 & 0.824 & 0.824 & 0.842 \\
& \texttt{E4M1} & 0.808 & 0.845 & 0.833 & 0.843 & 0.845 \\
& \texttt{E4M2} & 0.789 & 0.826 & 0.839 & 0.848 & 0.847 \\
\midrule
\textbf{SST-2} & \texttt{E4M0} & 0.869 & 0.862 & 0.862 & 0.897 & 0.897 \\
& \texttt{E4M1} & 0.876 & 0.873 & 0.878 & 0.901 & 0.908 \\
& \texttt{E4M2} & 0.860 & 0.895 & 0.897 & 0.906 & 0.906 \\
\bottomrule
\end{tabular}
}
\end{table}

\section{Practical Implications}

\subsection{Trading Precision for Batch Size on Edge Devices}

A key takeaway from our analysis is the explicit trade-off between quantization precision and mini-batch size. On edge hardware, there is often a significant imbalance between the availability of low-precision and high-precision compute units. For instance, Qualcomm’s Snapdragon X Elite offers up to 45 TOPS of \texttt{INT8} throughput via its Hexagon NPU, while its Adreno GPU provides 4.6 \texttt{FP32} TFLOPS \cite{qualcomm_snapdragon_x_elite}. This means low-precision compute exceeds high-precision compute by roughly an order of magnitude.

Practitioners can leverage this imbalance by using our framework. By accumulating gradients over several micro-batches (e.g., 4–8), one can emulate a larger effective batch size. This allows the use of abundant \texttt{INT8}/\texttt{INT4} MACs for the bulk of the computation. Our theorem guarantees that the extra stochastic variance introduced by the lower precision is compensated for, as it decays at a rate of $1/b$, ensuring that convergence remains stable.

\subsection{Hardware Overhead of Stochastic Rounding}

A natural concern is the hardware cost of implementing SR compared to the simpler RTN. While a fully precise SR unit can be expensive, practical approximations can be implemented with minimal overhead. To provide a concrete estimate, we analyzed the resource usage on an FPGA for converting from \texttt{FP32} to \texttt{FP8} (\texttt{E4M3}).

The mantissa rounding logic requires $9$ look-up tables (LUTs) for RTN. In contrast, an SR implementation using a $6$-bit linear feedback shift register (LFSR) as a pseudo-random number generator requires only $6$ extra registers and $9$ LUTs. When compared to the overall cost of a fused multiply-add (FMA) unit (e.g., an \texttt{E4M3} FMA requires $223$ LUTs and $140$ registers), the difference between SR and RTN is negligible. For a $16\times16$ systolic array ($256$ MACs), only $32$ rounding units are needed for the conversion, making the additional cost of SR insignificant in the context of the entire accelerator.

Furthermore, the feasibility of efficient SR implementation in hardware has been demonstrated in prior work. For example, \cite{zhang2022bfx} reports that their BFP converter, which includes the LFSR for SR, accounts for just $4.56\%$ of the chip area and $1.77W$ of power, compared to $47.79\%$ and $15.61W$ for the systolic array itself. This confirms that the benefits of SR can be realized without a significant hardware penalty.

\section{Conclusion}
\label{sec:conclusion}
This paper studied SR in mixed-precision training and provided a theoretical analysis of mini-batch SGD under quantization. We showed that the variance introduced by SR in activation and gradient quantization decays inversely with batch size, while bias from weight quantization remains unaffected. Empirical results confirmed that SR consistently outperforms deterministic RTN, particularly under aggressive quantization.

These results offer practical guidance for mixed-precision training: larger batch sizes can offset SR-induced noise, enabling more aggressive quantization without loss of convergence. This insight has direct implications for fine-tuning and training LLMs under resource constraints. Future work may examine adaptive batch sizing and broader quantization schemes.

\section*{Acknowledgments}

We thank the anonymous reviewers for their insightful comments and suggestions that helped improve this paper. We are also grateful to Ebby Samson for his help during the rebuttal on hardware cost modeling. We acknowledge the resources provided by AMD that made this research possible.

\section*{Limitations}
\label{sec:limitations}

While this work provides valuable insights into SR-based mixed-precision training, several limitations should be acknowledged:

\begin{itemize}
    \item \textbf{Quantization Schemes and Complementary Techniques:} Our analysis primarily centered on uniform quantization with SR. The interaction of SR and batch size with other advanced quantization techniques, such as non-uniform quantization, block-wise quantization (which we noted as complementary), or learned quantization, remains an area for future exploration. Similarly, a detailed investigation of the interplay with other mixed-precision optimization techniques like dynamic loss scaling or adaptive gradient clipping was beyond the scope of this paper.

    \item \textbf{Hardware Considerations and Performance Metrics:} Our experimental validation focused on convergence behavior (gradient norms, accuracy) rather than direct measurements of training speedup, memory footprint reduction, or energy consumption on target edge hardware. Such practical performance metrics are crucial for assessing the full benefits for edge deployment.
\end{itemize}

These limitations offer avenues for future research to build upon the foundational understanding of SR in low-precision training established in this work.

\bibliography{bibliography/reference}

\appendix

\onecolumn

\section{Proof of Lemma~\ref{lemma:bias_BW}}
\label{proof:lemma:bias_BW}
\begin{proof}
Assume $L(\bw)$ is $\mathcal{L}$-smooth, i.e.\ $\|\nabla L(\bw')-\nabla L(\bw)\|\le \mathcal{L}\|\bw'-\bw\|$ for all $\bw',\bw$.
Let $\hat\bw$ be obtained by elementwise uniform quantization with step $\Delta_W$. For RTN
\begin{equation}
|w_i-\hat w_i|\le \tfrac{\Delta_W}{2}\quad\forall i
\;\;\Rightarrow\;\;
\|\hat\bw-\bw\|^2=\sum_{i=1}^d (w_i-\hat w_i)^2 \le d\!\left(\tfrac{\Delta_W}{2}\right)^{\!2},
\end{equation}
hence $\|\hat\bw-\bw\|\le \tfrac{\sqrt d}{2}\,\Delta_W$.
By $\mathcal{L}$-smoothness,
\begin{equation}
\|\nabla L(\hat\bw)-\nabla L(\bw)\|\le \mathcal{L}\,\|\hat\bw-\bw\|\le \mathcal{L}\,\tfrac{\sqrt d}{2}\,\Delta_W.
\end{equation}
Therefore the bias bound is
\begin{equation}
B_W:=\|\nabla L(\hat\bw)-\nabla L(\bw)\|\;\le\; \mathcal{L}\,\tfrac{\sqrt d}{2}\,\Delta_W
\end{equation}

And for SR, $|w_i-\hat w_i|\le \Delta_W$, the bias bound is 

\begin{equation}
B_W \leq \mathcal{L}\sqrt d\Delta_W
\end{equation}

\end{proof}

\section{Proof of MSE Decomposition in Lemma~\ref{lem:q_mult_error_fully_quantized}}
\label{app:proof_lemma_error_decomp}
\begin{proof}
Let $\nabla W_{ij}^{(\text{true-mini})}=\frac{1}{b}\sum_{k\in\text{batch}} A_{ki}A^{\mathrm{out}}_{kj}$ and define
\begin{equation}
E \;=\; \nabla W_{ij}^{(\text{full})} - \widehat{\nabla W}_{ij}^{(\text{quant-mini})}
\;=\; S + Q',
\end{equation}
where $S=\nabla W_{ij}^{(\text{full})}-\nabla W_{ij}^{(\text{true-mini})}$ and
\begin{equation}
Q' \;=\; \nabla W_{ij}^{(\text{true-mini})}-\widehat{\nabla W}_{ij}^{(\text{quant-mini})}
= \frac{1}{b}\sum_{k\in\text{batch}}\!\Big(A_{ki}A^{\mathrm{out}}_{kj} - \widehat{A}_{ki}\,\widehat{A^{\mathrm{out}}}_{kj}\Big).
\end{equation}

Elementwise SR is unbiased and independent across elements/samples:
$\mathbb{E}[\widehat{A}\mid A]=A$, $\mathbb{E}[\widehat{A^{\mathrm{out}}}\mid A^{\mathrm{out}}]=A^{\mathrm{out}}$, and the errors are independent of $(A,A^{\mathrm{out}})$ given the values.

Conditioning on the realized mini-batch, $S$ is deterministic and
\begin{equation}
\mathbb{E}_{\epsilon}[Q'\mid \text{mini-batch}]
= \frac{1}{b}\sum_{k}\big(A_{ki}A^{\mathrm{out}}_{kj}-\mathbb{E}_{\epsilon}[\widehat{A}_{ki}\widehat{A^{\mathrm{out}}}_{kj}\mid\text{samples}]\big)
=0,
\end{equation}
because $\mathbb{E}[\widehat{A}_{ki}\widehat{A^{\mathrm{out}}}_{kj}\mid\text{samples}]=A_{ki}A^{\mathrm{out}}_{kj}$ by independence and zero-mean error.
Hence
\begin{equation}
\mathbb{E}[E^2]=\mathbb{E}[(S+Q')^2]=\mathbb{E}[S^2]+\mathbb{E}[(Q')^2]+2\,\mathbb{E}[SQ']
=\underbrace{\mathbb{E}[S^2]}_{T^{\mathrm{s}}_{ij}}+\underbrace{\mathbb{E}[(Q')^2]}_{T^{\bbQ}_{ij}},
\end{equation}
since $\mathbb{E}[SQ']=\mathbb{E}_{\text{samples}}[S\,\mathbb{E}_{\epsilon}[Q'|\text{samples}]]=0$.
\end{proof}

\section{Proof of Lemma~\ref{lem:quant_error_scaling}}
\label{app:proof_lemma_quant_error_scaling}
\begin{proof}
From Lemma~\ref{lem:q_mult_error_fully_quantized},
\begin{equation}
T^{\bbQ}_{ij}
=\mathbb{E}\!\left[\!\left(\frac{1}{b}\sum_{k}E_{k,ij}\right)^{\!2}\right],\qquad
E_{k,ij}=A_{ki}A^{\mathrm{out}}_{kj}-\widehat{A}_{ki}\,\widehat{A^{\mathrm{out}}}_{kj}.
\end{equation}
Let $X=A_{ki}$, $Y=A^{\mathrm{out}}_{kj}$ and $e_X=X-\widehat{X}$, $e_Y=Y-\widehat{Y}$. Then
\begin{equation}
E_{k,ij}=XY-\widehat{X}\widehat{Y}=X e_Y + Y e_X - e_X e_Y.
\end{equation}
Under the SR assumptions, $\mathbb{E}[e_X\mid X]=\mathbb{E}[e_Y\mid Y]=0$, $(e_X,e_Y)$ are independent across elements/samples and independent of $(X,Y)$; thus $E_{k,ij}$ are i.i.d., zero-mean across $k$, and
\begin{equation}
T^{\bbQ}_{ij}=\frac{1}{b}\,\sigma^2_{E_{ij}},\qquad
\sigma^2_{E_{ij}}=\mathbb{E}\big[(X e_Y + Y e_X - e_X e_Y)^2\big].
\end{equation}
Expanding and using the independence/zero-mean properties,
\begin{equation}
\sigma^2_{E_{ij}}
= \mathbb{E}[X^2]\,\underbrace{\mathbb{E}[e_Y^2]}_{\sigma_{A^{\mathrm{out}}}^2}
+ \mathbb{E}[Y^2]\,\underbrace{\mathbb{E}[e_X^2]}_{\sigma_{A}^2}
+ \sigma_{A}^2\,\sigma_{A^{\mathrm{out}}}^2.
\end{equation}
For uniform quantization with step sizes $\Delta_A,\Delta_{A^{\mathrm{out}}}$,
$|e_X|\le \Delta_A$ and $|e_Y|\le \Delta_{A^{\mathrm{out}}}$, hence
$\sigma_A^2\le \Delta_A^2$ and $\sigma_{A^{\mathrm{out}}}^2\le \Delta_{A^{\mathrm{out}}}^2$.
Therefore
\begin{equation}
T^{\bbQ}_{ij}
\;\le\; \frac{1}{b}\left(
\mathbb{E}[A_{ki}^2]\Delta_{A^{\mathrm{out}}}^2
+ \mathbb{E}[(A^{\mathrm{out}}_{kj})^2]\Delta_A^2
+ \Delta_A^2\,\Delta_{A^{\mathrm{out}}}^2
\right).
\end{equation}
If $\Delta\propto 2^{-B}$ (bitwidth $B$), then $T^{\bbQ}_{ij} \propto \tfrac{1}{b}\,2^{-2B}$ with constants depending on $\mathbb{E}[A_{ki}^2]$ and $\mathbb{E}[(A^{\mathrm{out}}_{kj})^2]$.
\end{proof}

\section{Proof of Theorem~\ref{thm:main_convergence}}
\begin{proof}
Let $L(\bw)=\mathbb{E}_{x,y}[L(\bw;x,y)]$ denote the expected loss. Assume:
(i) $L(\bw)$ is $\mathcal{L}$-smooth, i.e., $\|\nabla L(\bw')-\nabla L(\bw)\|\le \mathcal{L}\|\bw'-\bw\|$ for all $\bw',\bw$; 
(ii) $L(\bw)$ is bounded below by $L_{\min}$.
The SGD update is $\bw_{t+1}=\bw_t-\eta\,\tilde G(\bw_t)$ with constant step size $\eta>0$.
Define
\begin{equation}
\widehat{G}(\bw_t):=\mathbb{E}_t[\tilde G(\bw_t)],\qquad
V^2:=\frac{\sigma_S^2+\sigma_Q^2}{b},
\end{equation}
so that $\mathbb{E}_t[\|\tilde G(\bw_t)\|^2]\le V^2+\|\widehat{G}(\bw_t)\|^2$.

\paragraph{Step 1: $\mathcal{L}$-smoothness.}
By $\mathcal{L}$-smoothness and $\bw_{t+1}-\bw_t=-\eta \tilde G(\bw_t)$,
\begin{align}
\mathbb{E}_t[L(\bw_{t+1})]
&\le L(\bw_t)+\langle \nabla L(\bw_t),\mathbb{E}_t[\bw_{t+1}-\bw_t]\rangle
+\frac{\mathcal{L}}{2}\,\mathbb{E}_t[\|\bw_{t+1}-\bw_t\|^2]\nonumber\\
&= L(\bw_t)-\eta \langle \nabla L(\bw_t),\widehat{G}(\bw_t)\rangle
+\frac{\eta^2 \mathcal{L}}{2}\,\mathbb{E}_t[\|\tilde G(\bw_t)\|^2]\nonumber\\
&\le L(\bw_t)-\eta \langle \nabla L(\bw_t),\widehat{G}(\bw_t)\rangle
+\frac{\eta^2 \mathcal{L}}{2}\Big(V^2+\|\widehat{G}(\bw_t)\|^2\Big).
\label{eq:conv_step1}
\end{align}

\paragraph{Step 2: Bounds using the bias lemma.}
By Lemma~\ref{lemma:bias_BW}, $\|\widehat{G}(\bw_t)-\nabla L(\bw_t)\|\le B_W$.
\begin{align}
\langle \nabla L(\bw_t),\widehat{G}(\bw_t)\rangle
&= \|\nabla L(\bw_t)\|^2+\langle \nabla L(\bw_t),\widehat{G}(\bw_t)-\nabla L(\bw_t)\rangle \nonumber\\
&\ge \|\nabla L(\bw_t)\|^2 - \|\nabla L(\bw_t)\|\,\|\widehat{G}(\bw_t)-\nabla L(\bw_t)\| \nonumber\\
&\ge \|\nabla L(\bw_t)\|^2 - B_W\,\|\nabla L(\bw_t)\|,
\label{eq:inner_bound}
\end{align}
and
\begin{align}
\|\widehat{G}(\bw_t)\|^2
&= \|\nabla L(\bw_t) + (\widehat{G}(\bw_t)-\nabla L(\bw_t))\|^2 \nonumber\\
&\le \big(\|\nabla L(\bw_t)\|+\|\widehat{G}(\bw_t)-\nabla L(\bw_t)\|\big)^2 \nonumber\\
&\le \big(\|\nabla L(\bw_t)\|+B_W\big)^2
\;\le\; 2\|\nabla L(\bw_t)\|^2 + 2 B_W^2.
\label{eq:normsq_bound}
\end{align}

\paragraph{Step 3: Substitute \eqref{eq:inner_bound} and \eqref{eq:normsq_bound} into \eqref{eq:conv_step1}.}
We get
\begin{align}
\mathbb{E}_t[L(\bw_{t+1})]
&\le L(\bw_t) - \eta\Big(\|\nabla L(\bw_t)\|^2 - B_W\|\nabla L(\bw_t)\|\Big)
+\frac{\eta^2 \mathcal{L}}{2}\Big(V^2 + 2\|\nabla L(\bw_t)\|^2 + 2 B_W^2\Big)\nonumber\\
&= L(\bw_t) - \eta\|\nabla L(\bw_t)\|^2 + \eta B_W\|\nabla L(\bw_t)\|
+ \eta^2 \mathcal{L} \|\nabla L(\bw_t)\|^2 + \eta^2 \mathcal{L} B_W^2 + \frac{\eta^2 \mathcal{L}}{2} V^2.
\label{eq:after_subs}
\end{align}

\paragraph{Step 4: Rearrangement and Young's inequality.}
Move the gradient-norm terms to the left:
\begin{align}
\eta\|\nabla L(\bw_t)\|^2 - \eta^2 \mathcal{L}\|\nabla L(\bw_t)\|^2
&\le L(\bw_t) - \mathbb{E}_t[L(\bw_{t+1})]
+ \eta B_W\|\nabla L(\bw_t)\|
+ \eta^2 \mathcal{L} B_W^2 + \frac{\eta^2 \mathcal{L}}{2} V^2.\nonumber
\end{align}
Apply Young's inequality to the linear term with $\epsilon=1$:
\begin{equation}
\eta B_W\|\nabla L(\bw_t)\|
\;\le\; \frac{\eta}{2}\|\nabla L(\bw_t)\|^2 + \frac{\eta}{2} B_W^2.
\end{equation}
Thus,
\begin{align}
\big(\eta - \eta^2 \mathcal{L} - \tfrac{\eta}{2}\big)\|\nabla L(\bw_t)\|^2
&\le L(\bw_t) - \mathbb{E}_t[L(\bw_{t+1})]
+ \Big(\tfrac{\eta}{2} + \eta^2 \mathcal{L}\Big) B_W^2 + \frac{\eta^2 \mathcal{L}}{2} V^2.\nonumber
\end{align}
Equivalently,
\begin{equation}
\eta\Big(\tfrac12 - \eta \mathcal{L}\Big)\,\|\nabla L(\bw_t)\|^2
\;\le\;
L(\bw_t) - \mathbb{E}_t[L(\bw_{t+1})]
+ \eta\Big(\tfrac12 + \eta \mathcal{L}\Big) B_W^2
+ \frac{\eta^2 \mathcal{L}}{2} V^2.
\label{eq:key_before_sum}
\end{equation}

\paragraph{Step 5: Telescoping}
If $\eta \le \tfrac{1}{4\mathcal{L}}$, then $\tfrac12 - \eta \mathcal{L} \ge \tfrac14$, so the left side of \eqref{eq:key_before_sum} is at least $\tfrac{\eta}{4}\|\nabla L(\bw_t)\|^2$. Taking full expectation,
\begin{equation}
\frac{\eta}{4}\,\mathbb{E}[\|\nabla L(\bw_t)\|^2]
\le \mathbb{E}[L(\bw_t)] - \mathbb{E}[L(\bw_{t+1})]
+ \eta\Big(\tfrac12 + \eta \mathcal{L}\Big) B_W^2
+ \frac{\eta^2 \mathcal{L}}{2} V^2.
\end{equation}
Summing $t=0$ to $T-1$ and using telescoping plus $L(\bw_T)\ge L_{\min}$,
\begin{equation}
\frac{\eta}{4}\sum_{t=0}^{T-1}\mathbb{E}[\|\nabla L(\bw_t)\|^2]
\le L(\bw_0)-L_{\min} + T\,\eta\Big(\tfrac12 + \eta \mathcal{L}\Big) B_W^2
+ T\,\frac{\eta^2 \mathcal{L}}{2} V^2.
\end{equation}
Divide by $T$ and multiply by $4/\eta$:
\begin{equation}
\frac{1}{T}\sum_{t=0}^{T-1}\mathbb{E}[\|\nabla L(\bw_t)\|^2]
\le \frac{4\,(L(\bw_0)-L_{\min})}{\eta T}
+ (2 + 4\eta \mathcal{L})\,B_W^2
+ 2\eta \mathcal{L}\,V^2.
\end{equation}
Finally substitute $V^2=(\sigma_S^2+\sigma_Q^2)/b$ to obtain
\begin{equation}
\frac{1}{T}\sum_{t=0}^{T-1}\mathbb{E}[\|\nabla L(\bw_t)\|^2]
\;\le\; \frac{4\,(L(\bw_0)-L_{\min})}{\eta T}
+ (2 + 4\eta \mathcal{L})\,B_W^2
+ 2\eta \mathcal{L}\,\frac{\sigma_S^2+\sigma_Q^2}{b},
\end{equation}
which is the claimed bound with $C_B=2+4\eta \mathcal{L}$ and $C_V=2$.
\end{proof}

\section{Experiment setup}
\label{app:experiment_setup}
For CIFAR-10 experiments, we train a Wide ResNet-16-4 model on the CIFAR-10 dataset. A constant learning rate of $1 \times 10^{-4}$ is used, and standard data augmentation is disabled to enhance stability and isolate the quantization effects under investigation. Training proceeds for 20,000 steps, which is sufficient for the training loss to stabilize near zero across most configurations. As a measure of convergence quality near a stationary point, we evaluate the squared $\ell_2$ norm of the full-precision gradient, $|\mathbb{E}[\nabla_\theta L(\theta)]|^2$, computed at the end of each epoch after step 15,000.

For LMSYS-Chat experiments, we fine-tune a Llama-3.2-3B model on the first 1,000 conversations from the LMSYS-Chat dataset, using a next-token prediction objective in SFT format (assistant prompts masked out). We train for 800 steps using the Adam optimizer with a learning rate of $5 \times 10^{-5}$. The training loss reliably stabilizes near zero across configurations. Convergence is assessed via the squared $\ell_2$ norm of the full-precision gradient, computed at the end of epochs after step 400.

For downstream evaluation, we fine-tune Llama-3.2-3B on GSM8K for 100 steps with Adam (learning rate $5 \times 10^{-7}$, context length 512), and for 300 steps with vanilla SGD (learning rate $10^{-3}$). End-of-training accuracy is reported across different batch sizes. Additionally, we fine-tune BERT-base-uncased on the GLUE benchmark, training each task for 1,500 steps with Adam (learning rate $2 \times 10^{-5}$) and batch sizes in ${16, 32, 64, 128, 256}$. Task-specific metrics are used for evaluation: accuracy for most tasks, Matthews correlation for CoLA, and Pearson/Spearman average for STS-B.

Across all experiments, we evaluate multiple quantization formats (\texttt{E4M0}, \texttt{E4M1}, \texttt{E4M2}) applied to activations and gradients under both stochastic rounding (SR) and round-to-nearest (RTN). Training and fine-tuning are implemented in PyTorch 2.7 and executed on 2 NVIDIA H100 GPUs for around 1 week.

\section{Licenses, Models, and Datasets}
\label{sec:licenses_datasets}

This research makes use of publicly available models and datasets. Their sources and licenses are summarized below to ensure transparency and proper attribution.  

\begin{itemize}
    \item \textbf{Llama Model Family:}  
    Llama~3 models are released under the \textbf{Meta Llama~3 Community License Agreement}.  

    \item \textbf{CIFAR-10 Dataset:}  
    Created by Alex Krizhevsky, Vinod Nair, and Geoffrey Hinton; hosted by the University of Toronto (\url{https://www.cs.toronto.edu/~kriz/cifar.html}).  

    \item \textbf{LMSYS-Chat Dataset:}  
    The \texttt{lmsys-chat-1m} dataset is distributed under the \textbf{LMSYS-Chat-1M Dataset License Agreement}.  

    \item \textbf{GSM8K (Grade School Math 8K):}  
    Released by OpenAI and licensed under the \textbf{Apache2 License} (\url{https://huggingface.co/datasets/openai/gsm8k}).  

    \item \textbf{BERT (Bidirectional Encoder Representations from Transformers):}  
Released by Google AI Language and licensed under the \textbf{Apache License 2.0} (\url{https://github.com/google-research/bert}).

    \item \textbf{GLUE Benchmark:}  
    A collection of multiple NLP datasets, distributed under the \textbf{Apache 2.0 License} license (\url{https://www.tensorflow.org/datasets/catalog/glue}).  
\end{itemize}

All models and datasets were used in accordance with their respective licensing terms. All the models and datasets are used according to the intended usage. We thank the creators and maintainers of these resources for their contributions to the research community. 

We did not collect or use any data for this work. Therefore, no personally identifying information (PII) or offensive content is present, and no additional anonymization or protection steps were required. No datasets or artifacts were collected or produced as part of this work, so documentation of domains, languages, linguistic phenomena, or demographic groups is not applicable.

\section{Potential Risk}

The methods presented in this paper aim to enhance the efficiency of training Large Language Models (LLMs), particularly through improved understanding and application of Stochastic Rounding (SR) in mixed-precision settings. By potentially lowering computational and memory barriers, our work could contribute to the democratization of LLM development and fine-tuning, enabling smaller research groups, startups, or even on-device applications. This increased accessibility can foster innovation and allow for more diverse applications tailored to specific needs, potentially benefiting areas like personalized education or assistive technologies.

However, as with any technology that makes powerful AI more accessible, there are dual-use considerations. Lowering the technical threshold for training capable LLMs could inadvertently facilitate their misuse for generating disinformation, spam, or other harmful content if not accompanied by responsible development practices and robust safeguards. Furthermore, while our work focuses on training efficiency, it does not inherently address pre-existing biases within datasets or models. Therefore, the efficient training of LLMs must be paired with ongoing efforts in bias detection, mitigation, and the ethical deployment of these increasingly capable systems. We believe continued research into efficient and responsible AI development is crucial.

\section{Use Of AI Assistants}

In this work, we used AI coding and polishing for writing.

\end{document}